\documentclass{article}

\usepackage{microtype}
\usepackage{graphicx}
\usepackage{booktabs} 
\usepackage{subcaption}
\usepackage{hyperref}
\usepackage{multirow}
\usepackage{float} 

\usepackage[accepted]{icml2025}
\newcommand{\bandwidth}{L}
\newcommand{\reg}{\lambda}
\newcommand{\round}[1]{\left(#1\right)}
\newcommand{\Real}{\mathbb{R}}

\usepackage{amsmath}
\usepackage{amssymb}
\usepackage{mathtools}
\usepackage{amsthm}
\usepackage{amsfonts} 
  
\renewcommand{\P}{\ensuremath{\mathbb{P}}} 

\usepackage[capitalize,noabbrev]{cleveref}

\theoremstyle{plain}
\newtheorem{theorem}{Theorem}[section]
\newtheorem{proposition}[theorem]{Proposition}

\theoremstyle{definition}

\theoremstyle{remark}

\usepackage[textsize=tiny]{todonotes}

\newcommand{\bX}{ \mbox{\bf X}}

\newcommand{\bZ}{ \mbox{\bf Z}}

\newcommand{\bY}{ \mbox{\bf Y}}

\newcommand{\bM}{ \mbox{\bf M}}

\newcommand{\argmin}{{\mathop{\rm arg\, min}}}

\newcommand{\beq}{ \begin{equation}}
\newcommand{\eeq}{ \end{equation}}
\newcommand{\beqn}{ \begin{eqnarray}}
\newcommand{\eeqn}{ \end{eqnarray}}

\usepackage{caption}
\renewcommand{\arraystretch}{1.5}

\icmltitlerunning{Doubly Adaptive Neighborhood Conformal Estimation}

\begin{document}

\twocolumn[
\icmltitle{DANCE: Doubly Adaptive Neighborhood Conformal Estimation}




\begin{icmlauthorlist}
\icmlauthor{Brandon R. Feng}{NCSU,AMZN1}
\icmlauthor{Brian J. Reich}{NCSU}
\icmlauthor{Daniel Beaglehole}{UCSD}
\icmlauthor{Xihaier Luo}{BKHVN}
\icmlauthor{David Keetae Park}{BKHVN}
\icmlauthor{Shinjae Yoo}{BKHVN}
\icmlauthor{Zhechao Huang}{AMZN1}
\icmlauthor{Xueyu Mao}{AMZN1}
\icmlauthor{Olcay Boz}{AMZN2}
\icmlauthor{Jungeum Kim}{NCSU}
\end{icmlauthorlist}

\icmlaffiliation{NCSU}{Department of Statistics, North Carolina State University, Raleigh, NC, USA}
\icmlaffiliation{AMZN1}{Amazon.com, Seattle, WA, USA}
\icmlaffiliation{AMZN2}{Amazon.com, San Diego, WA, USA}
\icmlaffiliation{UCSD}{Department of Computer Science, University of California San Diego, San Diego, CA, USA}
\icmlaffiliation{BKHVN}{Computational Science Initiative, Brookhaven National Laboratory, Upton, NY, USA}

\icmlcorrespondingauthor{Jungeum Kim}{jkim255@ncsu.edu}


\vskip 0.3in
]



\printAffiliationsAndNotice 

\begin{abstract}
The recent developments of complex deep learning models have led to unprecedented ability to accurately predict across multiple data representation types. Conformal prediction for uncertainty quantification of these models has risen in popularity, providing adaptive, statistically-valid prediction sets. For classification tasks, conformal methods have typically focused on utilizing logit scores. For pre-trained models, however, this can result in inefficient, overly conservative set sizes when not calibrated towards the target task. We propose DANCE, a doubly locally adaptive nearest-neighbor based conformal algorithm combining two novel nonconformity scores directly using the data's embedded representation. DANCE first fits a task-adaptive kernel regression model from the embedding layer before using the learned kernel space to produce the final prediction sets for uncertainty quantification. We test against state-of-the-art local, task-adapted and zero-shot conformal baselines, demonstrating DANCE's superior blend of set size efficiency and robustness across various datasets.
\end{abstract}

\section{Introduction}

The stunning performance of large, pre-trained, multi-modal models, together with the excitement it has sparked, continues to fuel its quick adoption across many sectors of society. More specifically, vision-language models (VLMs) \cite{radford2021learning} have spread in use to diverse areas such as medical tasks \cite{bazi2023vision, luo2025vividmed, moon2022multi}, remote sensing \cite{hu2025rsgpt, zhang2024earthgpt, kuckreja2024geochat}, and misinformation detection \cite{cekinel2025multimodal, zhou2025m2, tahmasebi2024multimodal}. Yet despite these benefits, relatively few people view AI systems as fully trustworthy, in part because they often function as \textit{black boxes}. This gap between perceived unreliability and expected utility creates a persistent tension in widespread adoption. Uncertainty quantification of these black box outputs is a practical way to ease this tension: by explicitly acknowledging uncertainty in model responses and providing probabilistic assessments of output reliability.

To make black-box VLM predictions usable in practice, we need an uncertainty interface that (i) does \textit{not} require access to or retraining of the foundation model; and (ii) is statistically meaningful at the user’s chosen error tolerance. 
Conformal prediction (CP) provides a principled, model-agnostic approach for uncertainty quantification in this black-box setting: given a user-specified miscoverage level $\alpha$, CP wraps any predictor to produce a prediction set that contains the true label with probability at least $1-\alpha$~\cite{vovk2005algorithmic, vovk1999machine, lei2018distribution, shafer2008tutorial}. Additionally, many uncertainty quantification approaches require modeling assumptions about the data-generating distribution \cite{graves2011practical,gal2016dropout,hernandez2015probabilistic}. In contrast, conformal prediction is distribution-free, making it easy to deploy across a broad range of problems. The key assumption is exchangeability between the calibration dataset and the test dataset, where neither is involved in the training of the model. The prediction set for each test observation is then formed from appropriate calibration set responses based on some nonconformity score. The ease of usage has turned CP into a popular UQ tool for VLMs \cite{kostumov2024uncertainty, dutta2023estimating, silva2025conformal}.

\begin{figure*}[!t]
  \centering
  \includegraphics[width=\textwidth]{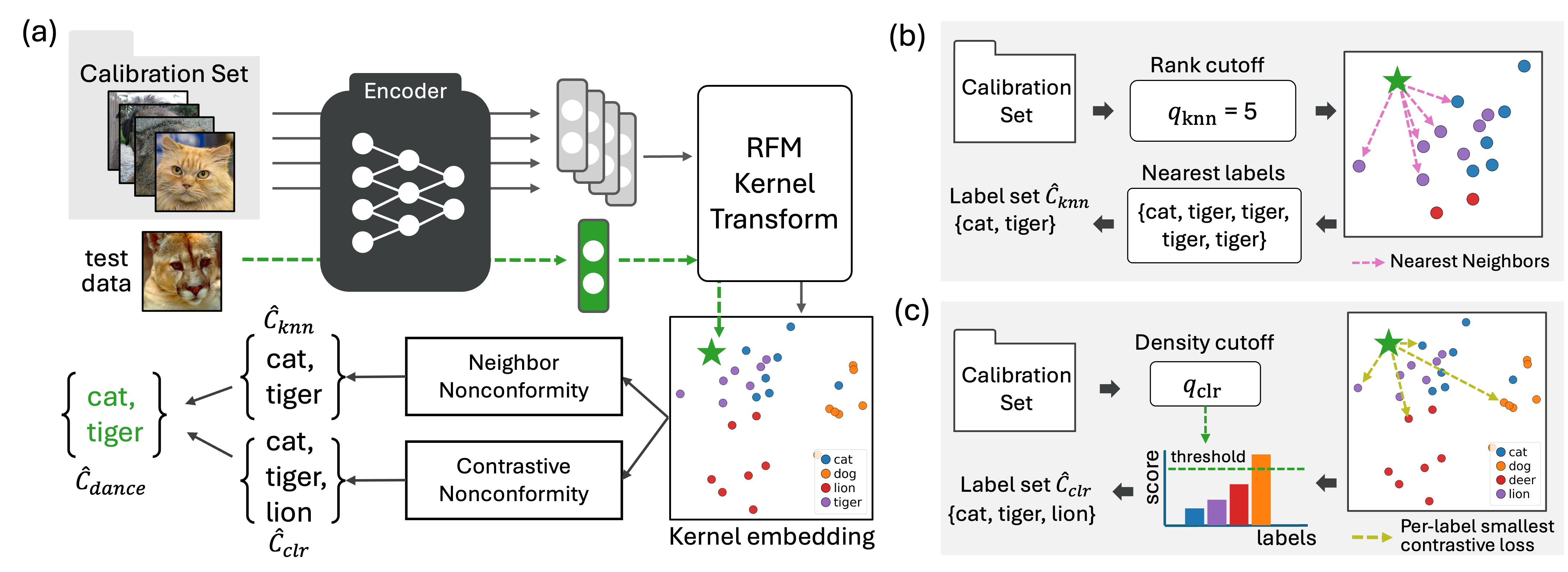}
  \caption{The DANCE pipeline. (a) The calibration set is used to compute nonconformity scores for a given test input. Our architectural contribution is the introduction of a task-adapted Recursive Feature Machines (RFM) kernel together with two complementary nonconformity scores. (b) The Neighbor Nonconformity module derives a rank cutoff $q_{\mathrm{knn}}$ in the kernel space from the calibration set. This cutoff is used to identify nearest labels and to construct the candidate label set $\hat{C}_{\mathrm{knn}}$ for the test input. (c) The Contrastive Nonconformity module defines a density cutoff $q_{\mathrm{clr}}$, below which a second candidate label set $\hat{C}_{\mathrm{clr}}$ is formed. The intersection of $\hat{C}_{\mathrm{knn}}$ and $\hat{C}_{\mathrm{clr}}$ yields the final conformal prediction set $\hat{C}_{\mathrm{dance}}$, yielding $1-\alpha$ coverage.}
  \label{fig:overview}
\end{figure*}

Yet CP’s distribution-free guarantee often comes with a cost: the resulting prediction sets can be unnecessarily large, especially when the base model’s confidence scores are misaligned with the downstream task~\cite{lei2018distribution,romano2019conformalized,romano2020classification,angelopoulos2020uncertainty}. For example, consider a simple three-way problem, cat vs. fox vs. dog, handled by an out-of-the-box VLM. If the model assigns similar scores to these semantically related classes, a conformal method may hedge by returning \{cat, fox, dog\} to maintain the desired $1-\alpha$ coverage: \textit{statistically valid, but operationally uninformative}.



This has motivated extensive efforts to improve efficiency. In classification, there has been significant focus on set size reduction through areas such as optimal thresholding from LAC \cite{sadinle2019least}, to novel nonconformity measures such as APS, RAPS and SAPS \cite{romano2020classification,angelopoulos2020uncertainty, huang2023conformal}, and optimal transport weighting \cite{silva2025conformal} to reduce out-of-distribution misalignment for zero-shot prediction. However, these methods utilize the logit probability layer of the model, losing potential impact from rich embedding information from earlier layers. 

Local variants of CP have offered a framework-based approach to set size reduction. These methods focus on the notion of local exchangeability, where calibration observations in a local space around the test point are favored in prediction set construction, resulting in smaller set sizes due to smaller response variation \cite{ding2023class, guan2023localized, mao2024valid, hore2025conformal}. For example, in a spatial regression setting, only calibration observations in a 2-D or 3-D radius around the test are used based on a decaying correlation parameter \cite{mao2024valid, jiang2024spatial, feng2025staci}. In classification, approaches such as Deep $k$-NN \cite{papernot2018deep} and CONFINE \cite{huang2024confine} introduce nonconformity scores based directly on the neighborhood label distribution. Both utilize cosine distance for neighbor retrieval. However, cosine distance is inherently linear, whereas the complex task-specific relationships may be better described in a non-linear kernel space. Meanwhile, approaches such as localized CP \cite{guan2023localized} and Neighborhood CP \cite{ghosh2023improving} re-weight existing nonconformity scores. A potential research gap of these methods is they use neighbor information only \emph{indirectly} as a scalar weight, rather than explicitly using neighbor distance in the nonconformity metric.

We propose the DANCE framework, combining two \emph{novel} locally adaptive nonconformity scores based on embedding distance in a fitted kernel-space to provide valid prediction sets for classification. DANCE contributes:
\vspace{-0.5em}
\begin{itemize}
\setlength\itemsep{0em}
    \item Computationally efficient alignment of the VLM image embedding space to learn a task-adapted kernel space
    
    \item Introduction of two nonconformity scores \textit{directly} using neighborhood information. DANCE utilizes one \textit{rank-based} neighborhood score and one \textit{density-based} contrastive score defined within the learned kernel space.
\end{itemize}

Rather than indirectly incorporating the kernel similarity as weights for existing scores, DANCE directly uses the kernel neighborhoods to calculate nonconformity. By constructing the nonconformity scores from embeddings adapted in the kernel space, DANCE \textit{maximizes} the utility of rich representations provided by pre-trained foundation models. This dual design blends set size efficiency for marginal coverage provided by the rank-based score with robustness for conditional coverage provided by the density-based score, yielding compact and reliable prediction sets.

\section{Preliminaries}

\textbf{Problem Statement.} Consider an image classification setting with images $X \in \mathcal{X}$ and labels $y \in \mathcal{Y} = \{1,...,c\}$. We use a fixed, pre-trained image encoder (e.g. CLIP \cite{radford2021learning}), $\phi:  \mathcal{X} \to \mathbb{R}^{d}$, mapping images to a $d$-dimensional embedding space. We are interested in constructing a prediction \emph{set}, $\hat{C}(.)\subseteq \mathcal{Y}$, that ensures marginal coverage at a user-specified error rate $\alpha \in (0,1)$ such that for test observation $(X^*, y^*)$, 
$$P(y^*\in\hat{C}(\phi(X^*))) \geq 1-\alpha.$$


\begin{figure*}[!t]
  \centering

  \begin{subfigure}[t]{0.49\textwidth}
    \centering
    \includegraphics[width=\textwidth]{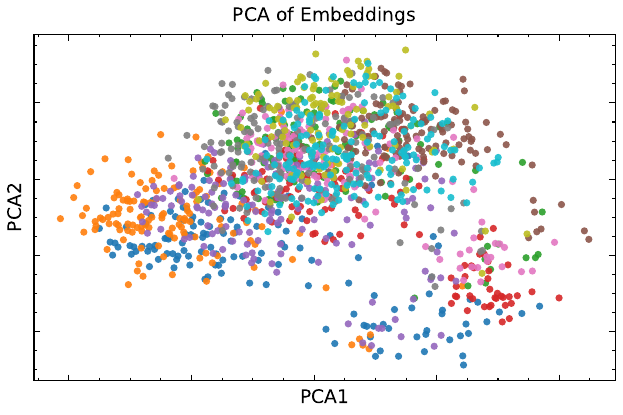}
    \caption{PCA of the original embedding space (top 10 classes; PC1 vs. PC2).}
    \label{fig:PCA1}
  \end{subfigure}
  \hfill
  \begin{subfigure}[t]{0.49\textwidth}
    \centering
    \includegraphics[width=\textwidth]{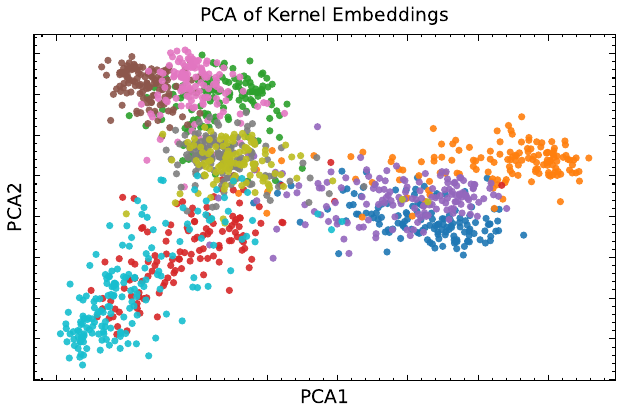}
    \caption{PCA of the kernel-transformed embedding space (top 10 classes; PC1 vs. PC2)}
    \label{fig:PCA2}
  \end{subfigure}

  \caption{Comparison of embedding projections for original and RFM kernel-transformed space for the top 10 most frequent classes (represented by colors) of the Imagenet-R dataset. Note the per-class clustering of the transformed kernel space. (a) Original embeddings; (b) Kernel-transformed embeddings.}
  \label{fig:Kernel_PCA}
\end{figure*}

\subsection{Split Conformal Prediction}

Conformal inference \citep{vovk2005algorithmic, shafer2008tutorial} allows us to construct prediction sets for black-box models with a finite-sample coverage guarantee. Consider calibration data (observations) $\mathcal{D}_{\rm cal} =\{(X_1,y_1),...,(X_n,y_n)\}$, where each $X_i \in \mathcal{X}$ is a feature and $y_i \in \mathcal{Y}$ is a response. For a new, unseen test point $(X_{n+1}, y_{n+1})$, we assume it is exchangeable with $\mathcal{D}_{\rm cal}$, i.e., their joint distribution is invariant under permutation. To apply split conformal prediction, we need a user-specified nonconformity score $S$ measuring the discrepancy between $(X_{n+1}, y_{n+1})$ and $\mathcal{D}_{\rm cal}$. Given a pre-specified confidence level $1-\alpha$, a nonconformity threshold, $q$, is calculated as the $\lceil (1-\alpha)(1+n) \rceil^{\text{th}}$ smallest score of $\{S_i\}_{i\in \mathcal{D}_{\rm cal}}$, where $S_i = S(X_i, y_i)$. Given a test observation $X_{n+1}$ and calibration set $\mathcal{D}_{\rm cal}$, the prediction set for $y_{n+1}$  is defined as $\hat{C}(X_{n+1}) = \{y\in \mathcal{Y} \mid S(X,y) \leq q\}.$ Assuming exchangeability holds between the test and calibration set, this prediction set guarantees $\mathbb{P}(y_{n+1} \in \hat{C}(X_{n+1})) \ge 1-\alpha$ \citep{lei2018distribution}. 



\subsection{Local Conformal Prediction}

 The framework of localized conformal prediction was introduced by \citet{guan2023localized}, creating more efficient prediction set sizes through weighting the nonconformity scores according to some localizer function centered at test observation $X_{n+1}$. Broadly, for localizer similarity $ K_{n+1,i} = K(X_{n+1}, X_i) \, \forall i\in1,...,n$, calculate probability mass $p_{n+1,i} = K_{n+1,i}/\sum_{j=1}^{n+1}K_{n+1,j}$ to weight scores $s_i$ before finding the quantile cutoff in SCP. For calibration neighbors with high dissimilarity, their influence in the nonconformity threshold calibration decreases significantly.

\citet{ghosh2023improving} proposed using an exponential kernel as the localizer such that for some function, $\phi(.)$ of the features,
$$K(X_i, X_j) = \exp \left( \frac{-||\phi(X_i) - \phi(X_j)||_2}{\bandwidth} \right),$$
for a tunable bandwidth parameter $\bandwidth$. They show this kernel localizer raises the set size efficiency for both regression and classification problems. For classification, intuitively, this local weighting ensures the threshold is primarily determined by neighbors with similar classes to the given test image, reducing the addition of unrelated classes to the prediction set. This idea of localized exchangeability raising set-size efficiency serves as a basis for our methodology.

\subsection{Recursive Feature Machines}

 Our method requires learning representative nearest neighbors for a given task, rather than using pre-trained similarities for identification. This is achieved using the Recursive Feature Machine (RFM) model \cite{radhakrishnan2024mechanism}, an efficient feature-learning method that composes a learned feature importance matrix and a kernel machine, as an adapter for the image embedding space of the VLM. 

We define notation for Kernel Ridge Regression (KRR) and RFM below. Define one-hot encoded training responses $\bY \in \mathbb{R}^{n \times c}$ where $n$ is the number of training observations and $c$ is the number of classes. Let input embeddings be $\bZ = \phi( \bX) \in \mathbb{R}^{n \times d}$ for the training image set, $\bX$, and test input embedding be $Z_{n+1} = \phi( X_{n+1}) \in \mathbb{R}^{1 \times d}$. The RFM is a learned kernel ridge regression defined as $$f(Z_{n+1}) = K_M(Z_{n+1},\bZ)\boldsymbol{\beta},$$ for a set of coefficients $\boldsymbol{\beta} \in \Real^{n \times c}$, obtained through KRR on the learned similarity measure $K_M$ from RFM. In this work, we define $K_M$ to be a generalized Laplace kernel function $$K_{M}(Z_1, Z_2) = \exp\round{-\round{\frac{\|Z_1-Z_2\|_M}{\bandwidth}}^{1/\xi}}~,$$ for bandwidth parameter $\bandwidth$, shape factor $\xi$, and feature importance matrix $\bM \in \mathbb{R}^{d \times d}$. Here, $\|z\|_M = (z \bM z^\top)^{1/2}$ indicates the Mahalanobis norm of $z$. The $\bM$ matrix learned by RFM is computed as the Average Gradient Outer Product (AGOP) of a KRR model, capturing the critical features in input space that most influence the prediction function and improving downstream prediction performance both provably and empirically \citep{zhu2025iteratively, beaglehole2025xrfm, radhakrishnan2025linear}. This importance matrix separates RFM from other KRR methods, providing an implicit measure of feature importance.

RFM alternates between updates of $\bM$ and $\boldsymbol{\beta}$, thus emphasizing embedding dimensions based on their utility for the prediction task. We additionally tune the $\bandwidth, \xi$ parameters on a validation set to optimize the kernel used to obtain nearest neighbors in DANCE. As we see in Figure \ref{fig:Kernel_PCA}, kernel PCA on the RFM kernel embedding space reveals distinct class-wise clustering. Additional details on AGOP computation, KRR and RFM are in Appendix \ref{app: krr}.

\section{Doubly Adaptive Neighborhood Conformal Estimation}

 The proposed conformal algorithm, DANCE, consists of a mixture of two novel nonconformity scores, \emph{directly} utilizing neighborhood information within a learned kernel space: a rank-based score (derived from neighborhood size) and a density-based score (derived from nearest-neighbor contrastive loss). 

Our design is motivated by a trade-off between \textit{efficiency} and \textit{robustness}. 
\vspace{-0.5em}
\begin{itemize}
\setlength\itemsep{0em}
    \item The \textbf{Rank-Based Score} ($S_{knn}$) relies on the relative ordering of neighbors. It is scale invariant, adapting to variations in embedding density by treating neighbors in sparse and dense regions equally. This adaptivity results in the global threshold producing highly size \textit{efficient} prediction sets. However, the performance may depend on the choice of distance metric, e.g., in this context, the kernel fit quality. In sparse regions, distant noisy neighbors can be treated as relevant for threshold determination.
    \item The \textbf{Density-Based Score} ($S_{clr}$) relies on contrastive similarity of neighbors. It is more \textit{robust}, aggregating embedding similarity in a region to identify all plausible classes, thus being able to better reach rare ones. However, this sensitivity to region density can yield larger set sizes, making it more conservative and less size efficient. On the other hand, the larger set size can increase conditional class-wise robustness.
\end{itemize}

In what follows, given pre-trained embedding function, $\phi$, let the embedded representation of an input image $X\in \mathcal{X}$ be denoted as $Z =\phi(X)$ and the corresponding labels be denoted $y \in \mathcal{Y}$ where $\mathcal{Y}$ is the set of all labels. Let the neighbor reference set be defined $\mathcal{D}_{\rm ref}=\{(\tilde{Z}_j,\tilde{y}_j)\}_{j=1}^{n_{\rm ref}}$ and the calibration set be defined $\mathcal{D}_{\rm cal}=\{(Z_j,y_j)\}_{j=1}^{n}$ where $\mathcal{D}_{\rm ref}$ is independent from $\mathcal{D}_{\rm cal}$. Define notation $\mathcal{N}_k(Z, \tilde{\mathbf{Z}}) \subseteq \mathcal{D}_{\rm ref}$ as the $k$-nearest neighbors to $Z$ from reference embeddings, $\tilde{\mathbf{Z}}$, in the learned RFM kernel space, $K_M(.,.)$. Nearest neighbors are identified as the $k$ reference set observations with minimum pairwise Euclidean distances after embeddings are projected by the learned feature matrix, $\bM^{1/2}$.





\subsection{$k$-NN Set Nonconformity}

We first define a prediction set (hereafter, the $k$-NN set classifier) that consists of \emph{all} unique labels appearing among the $k$ nearest neighbors i.e, 
\begin{equation}\label{eq:knn_C}
\hat{{L}}_k(Z) = {\rm Unique}\{\tilde{y}\in \mathcal{Y}\mid (\tilde{Z},\tilde{y}) \in \mathcal{N}_k(Z,\tilde{\textbf{Z}})\}.
\end{equation}
 The philosophy of the $k$-NN set classifier is based on the same premise of the standard $k$-NN classifier that the true label for $Z$ is likely to appear among the labels of its $k$-nearest neighbors. In the following, we conformalize the $k$-NN set, establishing its finite-set distribution-free theoretical coverage guarantee.

For feature-label pairings $(Z,y)$, and neighbor search number $m_{knn} \leq n_{\rm ref}$, we define the rank-based nonconformity score as 
\begin{equation}\label{eq:global_scoring}
    S_{knn}(Z,y) = \underset{k \in [m_{knn}] \cup \{\infty\}}{\argmin}\big\{y\in \hat{L}_{k}(Z)\big\},
\end{equation}
 where $\hat{L}_{\infty}(Z)$ is the set of all labels. Define the conformal threshold $q_{knn}$ by taking the $\lceil (1-\alpha_{knn})(1+n) \rceil^{\text{th}}$ smallest calibration score on $\mathcal{D}_{\rm cal}$ for some error rate $\alpha_{knn}$. The $k$-NN conformal prediction set is defined by the set of labels found within neighbor radius $q_{knn}$, i.e, 
\begin{equation}\label{eq:knn_pred_C}
  \hat{C}_{knn}(Z) = \hat{L}_{q_{knn}}(Z),  
\end{equation}
where $\hat{L}_k$ is defined in \eqref{eq:knn_C}. This conformalized prediction set comes with coverage $1-\alpha_{knn}$ as the following theorem.

\begin{theorem}\label{thm:knn_coverage}
Assume that the calibration sample $\mathcal{D}_{\rm cal}$ together with test observation $(Z_{n+1},y_{n+1})$
is exchangeable. For the prediction set $\hat{C}_{knn}(Z_{n+1})$ defined in \eqref{eq:knn_pred_C}, we have coverage guarantee 
\[  1-\alpha_{knn}\leq \mathbb{P}\left(y_{n+1}\in \hat{C}_{knn}(Z_{n+1}) \right)\leq 1-\alpha_{knn}+\frac{1}{n+1},\]
where the upper bound is established given there is no tie among the nonconformity scores and $q_{knn}\leq m_{knn}$.
\end{theorem}

\proof See, Appendix Section \ref{sec:knn_proof}

As the $S_{knn}$ score values are discrete, this can result in ties on the calibration set. In this case, the upper bound guarantee does not exist, which can lead to overcoverage, decreasing size efficiency. In order to have the upper bound guarantee as well, we add a small amount of noise $S_{knn}(Z,y) + U(Z,y)$ for $U \stackrel{iid}{\sim} \text{Unif}(0,\epsilon)$ for some $\epsilon < 1$.

The threshold $q_{knn}$ represents the maximum rank a label may take in a neighborhood to be considered plausible. Intuitively, this imposes a hard constraint on the diversity of labels, where only labels in the top ranks are predicted. For example, an input may only contain ``cat" candidate labels within the top ranks, ignoring potentially semantically related animals like ``tiger" and ``dog" in the embedding space, but which fail to reach the cutoff. This prioritizes \emph{size efficiency}, creating compact sets at the expense of potential robustness by excluding semantically related classes.



\subsection{Kernel Contrastive Nonconformity}

 Contrastive learning seeks to structure the embedding space by minimizing the distance between similar (positive) pairs of samples against a background set of dissimilar (negative) ones based only on the input space. The fundamental InfoNCE loss \cite{oord2018representation, sohn2016improved, wu2018unsupervised} was developed to achieve this structure during model training, using augmented variants of the same instance to produce positive pairs, while pushing away other instances. \citet{dwibedi2021little} extended this with Nearest Neighbor Contrastive Learning Representations (NNCLR). Rather than solely relying on augmentations, they build positive pairs using nearest neighbors from a support set, thus utilizing local density information. 

Our density-based conformity score is based on NNCLR. Let $Z_i$ be the embedded representation of our input with label $y_i$ and $Q(Z_i) = \mathcal{N}_{m_{clr}}(Z_i, \tilde{\textbf{Z}})$ be the $m_{clr}$ kernel nearest neighbors forming the support set. For the nearest-neighbor anchor, $A_i = \mathcal{N}_{1}(Z_i, \tilde{\textbf{Z}}\backslash Z_i)$, and embedded neighbor, $Z_j \in Q(Z_i)$, define $d_{\mathcal{K}}(A_i, Z_j) = 2\cdot(1 - K_M(A_i, Z_j))$. For each neighbor $Z_j \in Q(Z_i)$, first calculate pairwise loss:
\begin{equation}\label{eq:clr_loss}
    \mathcal{L}(Z_i, Z_j) = -\log \frac{\exp(-d_{\mathcal{K}}(A_i, Z_j) / \tau)}{\sum_{Z_k \in Q(Z_i)} \exp(-d_{\mathcal{K}}(A_i, Z_k) / \tau)}.
\end{equation}

Here, each support set neighbor is treated as the augmentation of $Z_i$.The denominator is summed only over the neighborhood $Q(Z_i)$ rather than the full reference set as in NNCLR. This results in a relative measure of \textit{local} density using our learned kernel space. Then, to conformalize this self-supervised loss, the labels are used to determine the required similarity between inputs for the desired outcome. The NNCLR-based nonconformity score is defined as the minimum loss among neighbors sharing the same label as $y_i$. This is formalized as
\begin{equation}\label{eq:nnclr_score}
     S_{clr}(Z_i, y_i) = \min \{\mathcal{L}(Z_i, Z_j) \mid Z_j \in Q(Z_i), y_j = y_i\}.
\end{equation}
Define the conformal threshold $q_{clr}$ by taking the $\lceil (1-\alpha_{clr})(1+n) \rceil^{\text{th}}$ smallest calibration score on $\mathcal{D}_{\rm cal}$ for some error rate $\alpha_{clr}$. Thus the CLR conformal prediction set with coverage $1-\alpha_{clr}$ is 
\begin{equation}\label{eq:clr_set}
    \hat{C}_{clr}(Z) = \{y \in \mathcal{Y} \mid S_{clr}(Z, y) \leq q_{clr}\}.
\end{equation}

Following directly from Theorems 2.1 and 2.2 in \citet{lei2018distribution}, the coverage, $\mathbb{P}\left(y_{n+1}\in \hat{C}_{clr}(Z_{n+1}) \right)$, can be bounded below by $1-\alpha_{clr}$ and above by $ 1-\alpha_{clr}+\frac{1}{n+1}$.

The threshold $q_{clr}$ represents the minimum relative similarity a candidate must share with an input compared to other local points. Intuitively, this captures representation grouping structure in the kernel space where semantically similar classes (i.e. cat, dog, tiger) may be predicted together, regardless of class imbalance. Thus, this score raises \emph{robustness} at the expense of size.



\subsection{DANCE}

We define the Doubly Adaptive Neighborhood Conformal Estimation (DANCE) prediction set as:
 
\begin{equation}\label{eq:dance}
  \hat{C}_{dance}(Z) = \hat{C}_{knn}(Z)\cap \hat{C}_{clr}(Z).  
\end{equation}

The combination is motivated by the complementary geometric properties of the two scores. Intuitively, by intersecting the irregular, discrete boundaries of the $k$-NN rank score with the smoother, continuous boundaries of the CLR density score, DANCE retains the semantic groupings captured by the latter while applying the former as a cardinality constraint to prune extraneous classes. Thus, controlling $q_{knn}$ is critical as it explicitly caps the maximum size of the DANCE prediction set. Assume the total error rate is set as $\alpha = \alpha_{knn}+\alpha_{clr}$. We introduce hyperparameter $\lambda$ to control the division such that $\alpha_{clr} = \lambda\cdot \alpha$ and $\alpha_{knn} = (1-\lambda)\cdot\alpha$. This leads to the following proposition for coverage guarantee.

\begin{proposition}\label{thm:dance}
\textit{Under the assumptions of Theorem \ref{thm:knn_coverage}, let $\hat{C}_{knn}(Z_{n+1})$ be the prediction set defined in \eqref{eq:knn_pred_C}. Additionally, let $\hat{C}_{clr}(Z_{n+1})$ be constructed under valid split conformal prediction conditions. For intersection set $\hat{C}_{dance}(Z_{n+1}) = \hat{C}_{knn}(Z_{n+1})\cap \hat{C}_{clr}(Z_{n+1})$, we have:}
\[\P(y_{n+1} \in \hat{C}_{dance}(Z_{n+1}))\geq 1-\alpha_{knn}-\alpha_{clr}.\] 
\end{proposition}

\proof See, Section \ref{sec:dance_proof}

Finally, Algorithm \ref{alg:nnclr_conformal_steps} found in Appendix \ref{app: dance_algo} depicts the construction of $\hat{C}_{dance}(Z)$.

\section{Experiments}

\subsection{Setup}

\subsubsection{Data Description}

DANCE's performance is evaluated across 11 common image classification datasets. For more \emph{generalized} datasets spanning a wide variety of subjects, we test the ImageNet variants of ImageNet \cite{deng2009imagenet}, ImageNet-R \cite{hendrycks2021many}, and ImageNet-Sketch \cite{wang2019learning}. For more \emph{focused} datasets on a single topic, we use EuroSAT \cite{helber2019eurosat}, StanfordCars \cite{krause20133d}, Food101 \cite{bossard2014food}, OxfordPets \cite{parkhi2012cats}, Flowers102 \cite{nilsback2008automated}, Caltech101 \cite{fei2004learning}, DTD \cite{cimpoi2014describing} and UCF101 \cite{soomro2012ucf101}. We only use the available test set of each dataset for this study.

Each dataset is randomly split into 40\% for RFM kernel fitting ($\mathcal{D}_{\text{sup}}$), 40\% for the calibration set ($\mathcal{D}_{\text{cal}}$) and 20\% for the testing set ($\mathcal{D}_{\text{Test}}$). For fitting the RFM model and kernel hyperparameters, the training set is split into 80\% for training and 20\% validation. For test conformal prediction results, each observation in the test set draws neighbors from the entire calibration set. 

Our methodology coverage guarantees apply to fully disjoint reference/calibration splits.  However, to improve data efficiency in finite-sample experiments, we also evaluate a calibration-reuse setting where $\mathcal{D}_{\text{cal}}$ is reused as reference set $\mathcal{D}_{\text{ref}}$. Specifically, when computing calibration nonconformity scores, for each $Z_i \in \mathcal{D}_{\rm cal}$, neighbors are found based via a leave-one-out search where $\mathcal{D}_{\text{ref}} = \mathcal{D}_{\text{cal}}\backslash \{Z_i\}$ to avoid self-matching. Similar reuse strategies have previously been used in neighbor-based methods \cite{papernot2018deep, ghosh2023improving}. Thus, we report empirical results under this reuse setting. Additionally, we empirically corroborate this choice against a theoretically-valid disjoint split in Appendix \ref{app: Disjoint_Ablation}, demonstrating \emph{closely} matching coverage and UQ performance. For fairness, all local CP methods compared reuse the $\mathcal{D}_{\text{cal}}$ as $\mathcal{D}_{\text{ref}}$. 


\begin{table*}[!t]
\centering
\small
\setlength{\tabcolsep}{4pt} 
\renewcommand{\arraystretch}{1.1}
\begin{tabular}{l c cccc c cccc}
\toprule
& & \multicolumn{4}{c}{\textbf{$\alpha = 0.1$}} & & \multicolumn{4}{c}{\textbf{$\alpha = 0.05$}} \\
\cmidrule(lr){3-6} \cmidrule(lr){8-11}
Method & Acc & Size & Cov & Cov Range & CCV & & Size & Cov & Cov Range & CCV \\
\midrule
DANCE (Ours)        & 0.856 & 2.49 & 0.930 & (0.903, 0.964) & 8.442 & & \underline{4.96} & 0.963 & (0.949, 0.986) & 5.364 \\
$k$-NN Set (Ours)      & 0.856 & 2.33 & 0.921 & (0.902, 0.956) & 8.877 & & \textbf{4.77} & 0.958 & (0.949, 0.971) & 5.785 \\
CLR Set (Ours)   & 0.856 & 5.82 & 0.936 & (0.913, 0.974) & \textbf{8.128} & & 8.20 & 0.965 & (0.955, 0.989) & \underline{5.346} \\
\midrule
Deep $k$-NN              & 0.856 & 4.26 & 0.936 & (0.905, 0.987) & 8.799 & & 7.50 & 0.967 & (0.951, 0.998) & \textbf{5.265} \\
RFM Adapter (APS)      & 0.856 & \underline{2.23} & 0.904 & (0.897, 0.917) & 8.817 & & 5.28 & 0.952 & (0.945, 0.957) & 5.948 \\
RFM Adapter (RAPS)     & 0.856 & \textbf{2.21} & 0.904 & (0.895, 0.919) & 8.854 & & 5.29 & 0.952 & (0.943, 0.957) & 5.979 \\
NCP (RAPS)             & 0.856 & 2.48 & 0.935 & (0.895, 0.974) & 8.846 & & 5.81 & 0.958 & (0.936, 0.983) & 5.721 \\
Conf-OT (APS)          & 0.720 & 6.04 & 0.902 & (0.889, 0.925) & \underline{8.436} & & 9.98 & 0.953 & (0.944, 0.964) & 5.846 \\
Conf-OT (RAPS)         & 0.720 & 5.29 & 0.903 & (0.890, 0.925) & 8.503 & & 8.10 & 0.953 & (0.948, 0.962) & 5.803 \\
\bottomrule
\end{tabular}
\caption{CLIP ViT-B/16: Average accuracy and UQ quality metrics (Set Size, Coverage, Coverage Range, Class-conditional Coverage Violation) across 11 datasets. Best and second best set size and CCV are \textbf{bolded} and \underline{underlined}, respectively.}
\label{tab:combined_metrics_VB16}
\end{table*}

\begin{table*}[!t]
\centering
\small
\setlength{\tabcolsep}{4pt} 
\renewcommand{\arraystretch}{1.1}
\begin{tabular}{l c cccc c cccc}
\toprule
& & \multicolumn{4}{c}{\textbf{$\alpha = 0.1$}} & & \multicolumn{4}{c}{\textbf{$\alpha = 0.05$}} \\
\cmidrule(lr){3-6} \cmidrule(lr){8-11}
Method & Acc & Size & Cov & Cov Range & CCV & & Size & Cov & Cov Range & CCV \\
\midrule
DANCE (Ours)      & 0.820 & 4.11 & 0.926 & (0.900, 0.957) & 8.746 & & \textbf{8.07} & 0.962 & (0.946, 0.976) & 5.627 \\
$k$-NN Set (Ours)      & 0.820 & 4.03 & 0.920 & (0.900, 0.961) & 9.006 & & \underline{8.47} & 0.957 & (0.950, 0.973) & 5.963 \\
CLR Set (Ours)   & 0.820 & 7.49 & 0.903 & (0.903, 0.970) & \underline{8.626} & & 9.96 & 0.965 & (0.945, 0.989) & \underline{5.430} \\
\midrule
Deep $k$-NN              & 0.820 & 5.66 & 0.935 & (0.904, 0.989) & 8.855 & & 11.58 & 0.972 & (0.957, 1.000) & \textbf{5.018} \\
RFM Adapter (APS)      & 0.820 & \underline{3.13} & 0.904 & (0.890, 0.915) & 8.761 & & 8.96 & 0.954 & (0.946, 0.963) & 5.982 \\
RFM Adapter (RAPS)     & 0.820 & \textbf{3.12} & 0.903 & (0.890, 0.913) & 8.796 & & 8.98 & 0.954 & (0.946, 0.963) & 6.002 \\
NCP (RAPS)             & 0.820 & 3.38 & 0.926 & (0.888, 0.958) & 8.904 & & 8.85 & 0.960 & (0.937, 0.984) & 5.660 \\
Conf-OT (APS)          & 0.661 & 8.44 & 0.907 & (0.893, 0.928) & \textbf{8.604} & & 13.95 & 0.954 & (0.948, 0.973) & 5.820 \\
Conf-OT (RAPS)         & 0.661 & 7.39 & 0.907 & (0.897, 0.926) & 8.651 & & 11.57 & 0.955 & (0.950, 0.981) & 5.859 \\
\bottomrule
\end{tabular}
\caption{CLIP ResNet-101: Average accuracy and UQ quality metrics (Set Size, Coverage, Coverage Range, Class-conditional Coverage Violation) across 11 datasets. Best and second best set size and CCV are \textbf{bolded} and \underline{underlined}, respectively.}
\label{tab:combined_metrics_RN101}
\end{table*}

\subsubsection{Implementation Details}

We use pre-trained CLIP models with two different backbones: ResNet-101, representing a CNN-based backbone, and ViT-B/16, representing a transformer-based backbone. For DANCE, the pre-trained image encoder is used to embed the images into the features used for kernel fitting and the conformal procedure. The RFM model serves as a task adapter and hyperparameter optimization of the kernel is done through a Bayesian optimization approach for 25 total evaluations, with accuracy as the objective. 

To determine the optimal ratio for $\lambda$, $\mathcal{D}_{\rm cal}$ is first split into 80\% calibration and 20\% validation. This is then used in a grid search over $\lambda \in [0,1]$ with a step size of 0.1. After selection, $\lambda$ is fixed and the final conformal thresholds are computed on the full $\mathcal{D}_{\rm cal}$. For all datasets, the neighborhood sizes are set at $m_{knn} = 100$ and $m_{clr} = 50$ and temperature $\tau = 0.01$. We provide results for $\alpha \in [0.1, 0.05]$. The impact of neighbor choice can be seen in Appendix \ref{app:appendix_ablation}. Details on $\lambda$ selection are given in Appendix \ref{app: lambda_sel}. In addition to DANCE, we also evaluate $k$-NN Set and CLR Set, representing DANCE with $\lambda$ fixed at $0$ and $1$, respectively.

We compare to the commonly-used nonconformity scores of APS \cite{romano2020classification} and RAPS \cite{angelopoulos2020uncertainty} based on perturbed logit rankings. To strictly isolate the contribution of the nonconformity scores of DANCE from the feature learning benefits of the RFM kernel, we utilize the same RFM-adapted backbone in two baseline frameworks. The first, \textbf{RFM Adapter}, represents a non-localized baseline where APS and RAPS are applied directly on the logits predicted by the RFM adapter. The second, \textbf{Neighborhood Conformal Prediction (NCP)} \cite{ghosh2023improving}, represents a similar localized approach where the RFM-predicted logits are re-weighted based on Euclidean distance in a tuned kernel space before being used in RAPS calculations. For NCP, we look at the nearest 50 neighbors as it is most similar to the $clr$ scoring format. 

Next, we also compare APS and RAPS used in \textbf{Conf-OT} \cite{silva2025conformal}, a state-of-the-art zero-shot approach using optimal transport to align the logit space with the calibration class distribution. For Conf-OT, we follow the original setup where CLIP logits are derived from image/text pairs using standard text templates \cite{gao2024clip, zhou2022learning}. We note that Conf-OT is the only method tested that uses the CLIP text encoder to construct class prototypes, where the other methods use only the image encoder. However, due to its zero-shot nature, prediction accuracy is lower as a result. For the RAPS score, hyperparameters $\lambda_{RAPS}$ and $k_{RAPS}$ are respectively set to 0.001 and 1 for the RFM Adapter and Conf-OT methods. In NCP, these are tuned in a grid search.

Finally, we also compare against a 1-layer variant of \textbf{Deep $k$-NN} \cite{papernot2018deep}, representing another method constructing the nonconformity score directly from neighbor information. Their nonconformity score is the proportion of neighbors with the correct label, given a neighborhood size. We again utilize the RFM-adapter and observations are extracted from the layer corresponding to the RFM kernel-transformed embedding space. We look at the 75 kernel nearest neighbors, as used in the original paper.

All experimental trials were conducted on a workstation equipped with an Intel Core i9-14900KF processor, 64 GB of 5200 MHz RAM, and an NVIDIA GeForce RTX 4090 GPU with 24 GB of VRAM. We utilize the open-source implementations provided by the authors for NCP and Conf-OT.

\begin{figure*}[t]
    \centering
    \begin{subfigure}[b]{0.32\textwidth}
        \includegraphics[width=\linewidth]{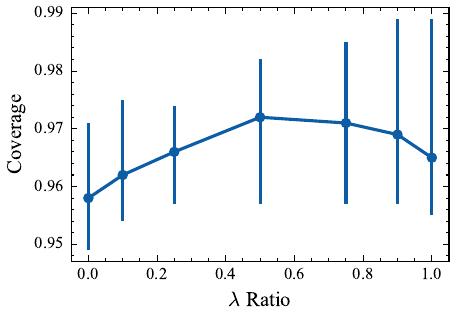}
        \caption{Average Coverage (Min/Max)}
    \end{subfigure}
    \hfill
    \begin{subfigure}[b]{0.32\textwidth}
        \includegraphics[width=\linewidth]{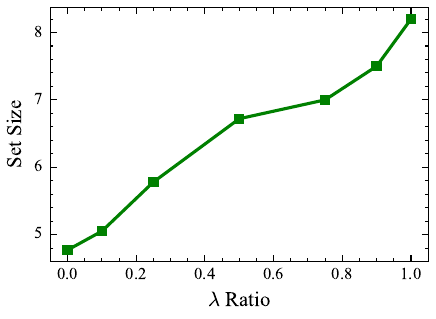}
        \caption{Average Set Size}
    \end{subfigure}
    \hfill
    \begin{subfigure}[b]{0.32\textwidth}
        \includegraphics[width=\linewidth]{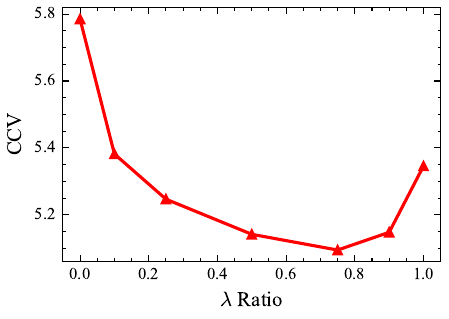}
        \caption{Average CCV}
    \end{subfigure}
    \caption{Ablation study of UQ based on $\lambda$. We observe the mix of $S_{knn}$ and $S_{clr}$ generally provides a trade-off between CCV and set size.}
    \label{fig:lambda_ablation}
\end{figure*}

\subsubsection{Performance Metrics}

The performance of all the methods are evaluated on both classification and UQ quality on the held test set. The Top-1 accuracy (Acc) is used for classification quality. Overall coverage (Cov) and set size (Size) are used as to measure marginal UQ quality, and class-conditional coverage violation (CCV) \cite{ding2023class} is used to measure conditional UQ quality. CCV is defined as 
$$CCV = 100 \times \frac{1}{|\mathcal{Y}|}\sum_{y\in \mathcal{Y}}|\hat{\text{Cov}}_y - (1-\alpha)|,$$
where $|\mathcal{Y}|$ is the total number of classes, and $\hat{\text{Cov}}_y$ is the coverage for class $y\in \mathcal{Y}$. While overall coverage places more weight on highly represented classes, CCV \emph{balances} the influence of large and small classes, thus being a metric reflecting \emph{robustness} to class distribution. We report average accuracy, coverage, set-size and coverage gap, as well as coverage range across all datasets. The results for individual datasets in the $\alpha = 0.05$ setting for CLIP ViT-B/16 are given in Appendix \ref{app: Detailed_ViT}. 

Finally, we note that DANCE is computationally efficient: the full $\lambda$ grid search required an average of 4.0 seconds (range: $0.2\text{s}$--$19.1\text{s}$), with mean total inference time taking an average of $0.08$s (range: $<1\text{ms}$--$0.33\text{s}$).

\subsection{Main Results}

Table \ref{tab:combined_metrics_VB16} shows results using embeddings from the CLIP ViT-B/16 backbone. At the relaxed error rate of $\alpha = 0.1$, RFM Adapter (RAPS) achieves the lowest average set size (2.21), with the APS variant following closely (2.23). CLR Set shows its robustness properties, with the lowest CCV (8.128), albeit with high set size (5.82), predicting over twice as many classes on average compared to more efficient baselines. DANCE achieves third best CCV (8.442), with comparable set size efficiency (2.49) to the top competitors. 

The benefits of DANCE are more pronounced for the more conservative $\alpha = 0.05$. The combination of properties is exhibited here where it achieves nearly the same conditional performance in robustness as CLR Set (5.364 vs 5.346), with the set size efficiency brought by the most efficient $k$-NN Set (4.96 vs 4.77). These account for the second and third best in CCV and the top two in set size. This has lower set size than the closest baseline of RFM Adapter with APS (5.28) and slightly higher CCV than the top Deep $k$-NN (5.265), highlighting DANCE's ability to optimize both efficiency and robustness without excelling in only one.

Table \ref{tab:combined_metrics_RN101} shows results using embeddings from the CLIP ResNet-101 backbone. The lower average accuracy (0.820 vs 0.856) makes uncertainty quantification more challenging, albeit with similar results. At $\alpha = 0.05$, DANCE shines, providing the lowest average set size (8.07) with third lowest CCV (5.627), further validating its capacity for high efficiency and conditional robustness. For both backbones, we note the high robustness of the Conf-OT and relative size efficiency, despite the lower zero-shot prediction accuracy, indicating benefits of using the text embeddings from the multi-modal model. Coverage range reveals undercoverage for some baselines, where individual dataset performance drops to 0.89 for $\alpha = 0.1$ and below 0.94 for $\alpha = 0.05$. In contrast, DANCE attains near-nominal empirical coverage across all experiments under this reuse setting.


\subsubsection{$\lambda$ Sensitivity}
Figure \ref{fig:lambda_ablation} shows the effect of $\lambda$ on the UQ metrics of coverage, set size and CCV for CLIP ViT-B/16 at $\alpha = 0.05$ based on $\lambda \in [0, 0.1, 0.25, 0.5, 0.75, 0.9, 1]$. We see the optimal coverage and CCV generally lies between the $knn$ and $clr$ scores on their own, with spikes in average CCV at both ends and an increasing average set size as $\lambda$ favors $clr$. This indicates the need for the grid search step to find the balance. DANCE also maintains marginal coverage for all $\lambda$ values.

\section{Conclusion}

In this work, we introduce DANCE (Doubly Adaptive Neighborhood Conformal Estimation), a novel framework that provides valid uncertainty quantification for pre-trained vision-language models. By operating within an efficiently learned kernel space rather than fixed pre-trained embeddings, DANCE effectively captures complex non-linear relationships between calibration and test set that is leveraged in set construction while also maximizing utility of the descriptiveness these embeddings provide.

Our key methodological contribution lies in the synergistic blend of two novel neighbor-based nonconformity scores: a rank-based score, $S_{knn}$, that ensures efficiency through a hard size cutoff, and a density-based score, $S_{clr}$, that promotes robustness through returning semantically-related groups. Through extensive experiments on CLIP adaptation, we demonstrate that this dual-score approach, optimized via a learned mixing parameter $\lambda$, consistently achieves a high balance of tight prediction sets and high conditional coverage validity compared to state-of-the-art baselines. 

We note there are limitations of DANCE. First, the efficiency and robustness hinges on the fit of the kernel. If the RFM kernel is a poor fit, the local embedding clustering can be diffuse, causing inefficient set sizes with high cutoff for $q_{knn}$ and the inclusion of overly robust neighborhoods from $q_{clr}$. The final set efficiency and robustness also depends on choice of $\lambda$, $m_{knn}$ and $m_{clr}$, making tuning an essential part. Second, in limited data settings, the reference and calibration sets may need to be reused, losing the strict theoretical coverage guarantees provided. Nevertheless, we see that DANCE maintains empirical coverage validity across all experiments and an ablation showed no significant deviation between this heuristic split and a valid disjoint split. Future work can extend this method to both continuous and generative settings, with small outcome-dependent changes to the nonconformity scores. The performance of the contrastive score may also benefit in a true multi-modal setting, where more embeddings sources are used to calculate semantic similarity. The robust performance of the zero-shot Conf-OT baseline against task-adapted competitors previews the predictive power multi-modal representations can bring.

\section*{Acknowledgements}

We thank Mikhail Belkin, Qi Zhao and Zhanlong Qiu for invaluable discussions on RFM, contrastive learning and kernel nearest neighbors.

\bibliographystyle{icml2025}
\bibliography{ref}

\newpage
\appendix
\onecolumn

\section{DANCE Algorithm} \label{app: dance_algo}

\begin{algorithm}[!h]
\caption{DANCE Conformal Label Sets}
\label{alg:nnclr_conformal_steps}
\begin{algorithmic}[1]
\STATE \textbf{Input:} Reference embeddings $\mathbf{Z}_{ref}$, Calibration embeddings $\mathbf{Z}_{Cal}$, calibration labels $\mathbf{y}_{Cal}$, test embedding $\mathbf{Z}_{n+1}$, learned Kernel $K_M(.)$, total error $\alpha$, error allocation parameter $\lambda$
\STATE \textbf{Output:} Prediction sets $\{\widehat C_i\}_{i=1}^{n_{test}}$
\vspace{0.3em}
\STATE \textbf{Calibration: compute scores and threshold}
\FOR{$i = 1$ \textbf{to} $n_{cal}$}
    \STATE Find all neighbor indices $\mathbf{I}_{Cal}$ from $K_M(\mathbf{Z}_{Cal}, \mathbf{Z}_{ref})$
    \STATE Compute $S_{knn}$ from $Z_{Cal,i}$ to its $m_{knn}$ neighbors indexed by $\mathbf{I}_{Cal, knn}[i,:]$
    \STATE Compute $S_{clr}$ from $Z_{Cal,i}$ to its $m_{clr}$ neighbors indexed by $\mathbf{I}_{Cal, clr}[i,:]$
\ENDFOR

\STATE Calculate $\alpha_{knn}, \alpha_{clr}$ based on $\lambda$ split of $\alpha$
\STATE Find $q_{knn}, q_{clr}$ from $1-(\alpha_{knn},\alpha_{clr})$ quantiles

\vspace{0.3em}
\STATE \textbf{Test: build label set}
\STATE Find all neighbor indices $\mathbf{I}_{n+1}$ from $K_M(\mathbf{Z}_{n+1}, \mathbf{Z}_{ref})$
\STATE Compute $S_{knn}$  from $Z_{n+1}$ to its $m_{knn}$ neighbors indexed by $\mathbf{I}_{n+1, knn}$
\STATE Build $\hat{C}_{knn}(Z_{n+1}) = \hat{\mathcal{Y}}_{q_{knn}}(Z_{n+1})$
\STATE Compute $S_{clr}$ from $Z_{n+1}$ to its $m_{clr}$ neighbors indexed by $\mathbf{I}_{n+1, clr}$
\STATE Build $\hat{C}_{clr}(Z_{n+1})$ from neighbor labels

\STATE \textbf{Return:} $\hat{C}_{dance}(Z_{n+1}) = \hat{C}_{knn}\cap \hat{C}_{clr}$
\end{algorithmic}
\end{algorithm}

\section{Technical details of Recursive Feature Machines (RFMs)}
\label{app: krr}

RFMs are an iterative algorithm for learning features with general machine learning models by learning representations through the Average Gradient Outer Product (AGOP) \citep{radhakrishnan2024mechanism}. RFMs are initialized with a base prediction algorithm, training inputs and labels, and validation inputs and labels. RFMs fit the training data through $T$ iterations of the following steps. RFM (1) first fits the base algorithm to the training data, (2) computes the AGOP of the fit algorithm from step 1, (3) transforms the data with the AGOP matrix, then (4) returns to step 1. The iteration proceeds for $T$ steps, each step constructing a separate predictor. The best predictor among the iterations is selected on the validation set. In our case, as is typically done, we perform the RFM update with Kernel Ridge Regression (KRR) as the base model. We briefly explain the KRR model. 

\paragraph{KRR} Let $X \in \mathbb{R}^{n \times d}$ denote training samples with ${x^{(i)}}^T$ denoting the sample in the $i$\textsuperscript{th} row of $X$ for $i \in [n]$ and $y \in \mathbb{R}^{n \times c}$ denote training labels, where $c$ is the number of output channels (e.g. one-hot encoded classes for $c>2$ classes). In our case, the inputs are themselves embeddings obtained from the internal activations of parent models such as CLIP or ResNet. Let $K: \mathbb{R}^{d} \times \mathbb{R}^{d} \to \mathbb{R}$ denote a kernel function (a positive semi-definite, symmetric function), such as the generalized Laplacian kernel we deploy in DANCE.  Given a ridge regularization parameter $\reg \geq 0$, KRR solved on the data $(X, y)$ gives a predictor, $\hat{f}: \mathbb{R}^{d} \to \mathbb{R}^{c}$, of the form: 
\begin{align}
\label{eq: kernel ridge regression}
    \hat{f}(x) = K(x, X) \beta~,
\end{align}
where $\beta$ is the solution to the following linear system:
\begin{align}
\qquad (K(X, X) + \lambda I) \beta = y~.
\end{align}
Here the notation $K(x, X) \in \mathbb{R}^{1 \times n}$ is the $n$-dimensional row vector with $K(x, X)_{i} = K(x, x^{(i)})$ and $K(X, X) \in \mathbb{R}^{n \times n}$ denotes the kernel matrix of pair-wise kernel evaluations $K(X, X)_{ij} = K(x^{(i)}, x^{(j)})$. The advantage of kernel functions in the context of this work is that the predictor admits a closed form solution, which can be robustly computed and generally has fast training times for datasets under $70$k samples. 

\paragraph{AGOP}
The $M$ matrix learned by RFM is computed as the Average Gradient Outer Product (AGOP) of a KRR model. For a given trained predictor $g: \Real^d \rightarrow \Real^c$ and its training inputs $x_1, \ldots, x_n \in \Real^d$, the AGOP operator produces a matrix $$\mathrm{AGOP}(g, \{x_i\}_{i=1}^n) = \frac{1}{n}\sum_{i=1}^n \nabla g(x_i) \nabla g(x_i)^\top \in \Real^{d \times d}~,$$ where $\nabla g \in \Real^{d \times c}$ is the (transposed) input-output Jacobian of the function $g$. The AGOP matrix of a predictor captures the critical features in input space that most influence the prediction function, provably and empirically improving downstream prediction performance when used as a feature extractor \citep{zhu2025iteratively, beaglehole2025xrfm, radhakrishnan2025linear}. 
\paragraph{RFMs}

We outline the particular RFM algorithm used in our work below.  Given that the procedure is the same for any block, we simplify notation by omitting the block subscripts $\ell$.  For any $x, z \in \mathbb{R}^{k}$, choice of \textit{bandwidth} parameter $L > 0$, and a positive semi-definite, symmetric matrix $M \in \mathbb{R}^{k \times k}$ and $K_M$ Mahalanobis kernel function, let $K_{M}(Z, z) \in \mathbb{R}^{n}$ such that $K_{M}(Z, z)_{i} =  K_{M}(a^{(i)}, z)$ and let $K_{M}(Z, Z) \in \mathbb{R}^{n \times n}$ such that $K_{M}(Z, Z)_{i, j} = K_{M}(a^{(i)}, a^{(j)})$.  Letting $M_0 = I$ and $\lambda \geq 0$ denote a regularization parameter, RFM repeats the following two steps for $T$ iterations:
\begin{align}
    &\text{Step 1:}~~ \hat{f}_t(z) = \beta_t K_{M_t}(Z, z)  ~\text{where}~ \beta_t = y [K_{M_t}(Z, Z) + \lambda I ]^{-1}~, \tag{Kernel ridge regression}\\
    & \text{Step 2:}~~ M_{t+1} = \frac{1}{n}\sum_{i=1}^{n} \nabla_z \hat{f}_t(a^{(i)}) \nabla_z \hat{f}_t(a^{(i)})^\top  \label{eq: agop} \tag{AGOP}~;
\end{align}
where $\nabla_z \hat{f}_t(a^{(i)})$ denotes the gradient of $\hat{f}_t$ with respect to the input $z$ at the point $a^{(i)}$.

We utilize the xRFM repo for the RFM implementation \citep{beaglehole2025xrfm}

\section{Proof of Theorems}

\subsection{Proof of Theorem \ref{thm:knn_coverage}}\label{sec:knn_proof}


\begin{proof}
Recall the definition of set-valued $k$NN classifier: 
\begin{equation*}
\hat{{L}}_k(Z) = {\rm Unique}\{\tilde{y}\in \mathcal{Y}\mid \tilde{y}\in \mathcal{N}_k(Z, \tilde{\mathbf{Z}})\},
\end{equation*}
where the neighborhood is searched over a reference dataset, $\mathcal{D}_{\rm ref}$, that is independent of
$\{\mathcal{D}_{\rm cal},(Z_{n+1},y_{n+1})\}$. Here, we add a bit of a technical tool. Assume the classifier model returns the set of prediction sets, \[f(x) = \{\hat{L}_1(Z),...,\hat{L}_{m_{knn}}(Z)\}\cup\{\hat{L}_\infty(Z)\},\]
where $\hat{L}_\infty(Z) = \{1,..,c\}$, the set of all labels. This function is available in the setting of the theorem, since it assumes the availability of $\hat{L}_k(Z)$ for all $k\in[m_{knn}]\cup\{\infty\}$. Then the nonconformity score function measures the quality of this $f(Z)$ by finding the minimum $k$ that achieves $y\in \hat{L}_k(Z)$. For feature-label pairings $(Z,y)$, we then define the rank-based nonconformity score that is implicitly defined on $f(Z)$ as 
\[ S_{knn}(Z,y) = \min\big\{k\in [m_{knn}]\cup\{\infty\}: y\in \hat{L}_k(Z)\big\},\]
where $\hat{L}_\infty(Z)=\mathcal{Y}$ is the set of all labels. Since the neighborhood search is based on an independent reference dataset, not the calibration dataset, the function $f$ is not associated with the calibration and test datasets, which does not negatively affect the assumed exchangeability.  Then, we define the conformal threshold $q_{knn}$ by taking the $\lceil (1-\alpha_{knn})(1+n) \rceil^{\text{th}}$ smallest score of the calibration scores on $\mathcal{D}_{\rm cal}$ for some error rate $\alpha_{knn}$. By the standard split conformal prediction argument, we have 
\begin{equation}\label{eq:snn}
1-\alpha\leq \P( S_{knn}(Z_{n+1}, y_{n+1})\leq q_{knn})\leq 1-\alpha+\frac{1}{n+1},
\end{equation}
where the upper bound establishes when there is no tie among the nonconformity scores of the calibration dataset and the test data point. 

The core part in the proof is an interesting observation of inclusion monotonicity that for $\tilde{y}\in \mathcal{Y}$, 
\begin{equation}\label{eq:order2}
\tilde{y}\not\in \hat{L}_k(Z) \Leftrightarrow \tilde{y}\not \in \hat{L}_j(Z)~ (\forall j\leq k).
\end{equation}
Therefore, \emph{given $q_{knn}<\infty$,} there is equivalence among events such as 
\[\{S_{knn}(Z_{n+1}, y_{n+1}) > q_{knn} \} =\cap_{k=1}^{q_{knn}} \{ y_{n+1}\not\in \hat{L}_k(Z_{n+1})\}  = \{ Z_{n+1}\not\in \hat{L}_{q_{knn}}(Z_{n+1})\}  ,\]
where the last equality is by \eqref{eq:order2}.
Therefore, we have $\{S_{knn}(Z_{n+1}, y_{n+1}) \leq  q_{knn} \} = \{ y_{n+1}\in \hat{L}_{q_{knn}}(Z_{n+1})\} .$
Recall the definition of the prediction set $ \hat{C}_{knn}(Z) = \hat{L}_{q_{knn}}(Z) $. Therefore,  the bound in \eqref{eq:snn} transfers to 
\begin{equation*}
1-\alpha\leq \P( y_{n+1}\in  \hat{C}_{knn}(Z_{n+1}))\leq 1-\alpha+\frac{1}{n+1}.
\end{equation*}
\end{proof}

\subsection{Proof of Proposition 1:}\label{sec:dance_proof} 


\textbf{Proof:} It follows from Boole's inequality that:
\begin{align*}
  \P(y_{n+1} &\not \in \hat{C}_{knn}(Z_{n+1}) \cup  y_{n+1}\not\in\hat{C}_{clr}(Z_{n+1}))\\ &\leq  \P(y_{n+1}\not \in \hat{C}_{knn}(Z_{n+1}))+\P(y_{n+1}\not \in \hat{C}_{clr}(Z_{n+1}))\leq \alpha_{knn}+\alpha_{clr}   
\end{align*}

where bound $\P(y\not \in \hat{C}_{knn}(Z_{n+1})) \leq \alpha_{knn}$ follows from \ref{thm:knn_coverage} and split conformal bound $\P(y\not \in \hat{C}_{clr}(Z_{n+1})) \leq \alpha_{clr}$.

\clearpage

\section{Disjoint Reference Ablation Study} \label{app: Disjoint_Ablation}

In the main text, we used a leave-one-out strategy where the calibration set was both reused as the reference set and the set used to search for $\lambda$ to maximize data efficiency. In this section, we validate those results against a theoretically valid disjoint split. More specifically, we define the reference set $\mathcal{D}_{\rm ref}$ as the training set $\mathcal{D}_{\rm sup}$ used to fit the RFM adapter. Furthermore, we evenly split the available calibration data, $\mathcal{D}_{\rm cal}$, into two disjoint subsets: $\mathcal{D}_{\lambda}$, used to select the mixing parameter $\lambda$, and $\mathcal{D}_{\rm Thr}$, used to compute the conformal thresholds $(q_{knn}, q_{clr})$.

Uncertainty Quantification (UQ) quality is compared between the reuse and disjoint strategies through set size, marginal coverage and class-conditional coverage violation (CCV). We evaluate with $m_{knn} = 100$, $m_{clr}=50$, $\lambda \in [0,0.1,...,0.9,1.0]$ and $\alpha \in [0.1, 0.05]$ on the ImageNet, ImageNet-Rendition and ImageNet-Sketch datasets.

The results in Table \ref{tab:ablation_disjoint} show that the two splits are very consistent to each other between both $\alpha$ levels in all three datasets. The disjoint split attains slightly lower CCV in all settings, indicating better conditional robustness than the reuse split. The set size efficiency difference for both splits is similar across all settings. While theoretical validity is important, we believe the close results here indicate the heuristic reuse split can be used to maximize data efficiency while maintaining the desired theoretical coverage properties DANCE provides. 

\begin{table}[h]
\centering
\small
\setlength{\tabcolsep}{6pt}
\begin{tabular}{lccccc}
\toprule
Dataset & $\alpha$ & $\mathcal{D}_{\rm ref}$ & Set Size & Coverage & CCV \\
\midrule
\multirow{4}{*}{ImageNet} 
  & \multirow{2}{*}{0.05} & Reuse & 10.75 & 0.953 & 6.450 \\
  & & Disjoint & 10.09 & 0.953 & 6.430 \\
  \cmidrule{2-6}
  & \multirow{2}{*}{0.10} & Reuse & 5.01 & 0.913 & 8.770 \\
  & & Disjoint & 4.55 & 0.914 & 8.760 \\
\midrule
\addlinespace
\multirow{4}{*}{ImageNet-R} 
  & \multirow{2}{*}{0.05} & Reuse & 6.30 & 0.951 & 4.627 \\
  & & Disjoint & 6.15 & 0.955 & 4.525 \\
  \cmidrule{2-6}
  & \multirow{2}{*}{0.10} & Reuse & 3.16 & 0.914 & 6.955 \\
  & & Disjoint & 3.44 & 0.924 & 6.824 \\
\midrule
\addlinespace
\multirow{4}{*}{ImageNet-Sketch} 
  & \multirow{2}{*}{0.05} & Reuse & 17.78 & 0.954 & 6.032 \\
  & & Disjoint & 17.86 & 0.958 & 5.927 \\
  \cmidrule{2-6}
  & \multirow{2}{*}{0.10} & Reuse & 6.42 & 0.903 & 8.625 \\
  & & Disjoint & 8.12 & 0.924 & 8.133 \\
\bottomrule
\end{tabular}
\caption{Comparison of reused (Main Paper) vs. disjoint reference set for UQ quality metrics( Set Size, Coverage and Class-conditional Coverage Violation) for the 3 ImageNet variants. }
\label{tab:ablation_disjoint}
\end{table}

\clearpage

\section{Neighbor Ablation Studies}
\label{app:appendix_ablation}

\begin{figure*}[h!]
    \centering

    \begin{subfigure}[b]{0.32\textwidth}
        \includegraphics[width=\linewidth]{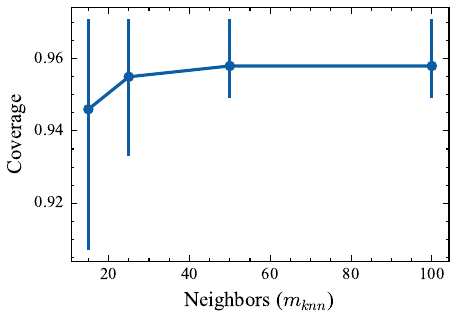}
        \caption{$m_{knn}$: Coverage}
    \end{subfigure}
    \hfill
    \begin{subfigure}[b]{0.32\textwidth}
        \includegraphics[width=\linewidth]{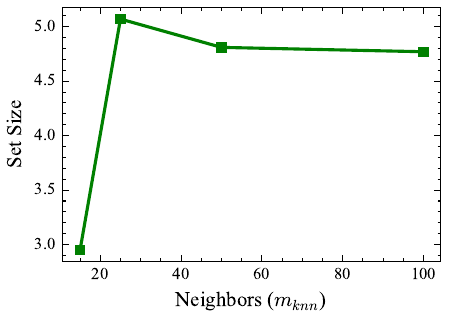}
        \caption{$m_{knn}$: Set Size}
    \end{subfigure}
    \hfill
    \begin{subfigure}[b]{0.32\textwidth}
        \includegraphics[width=\linewidth]{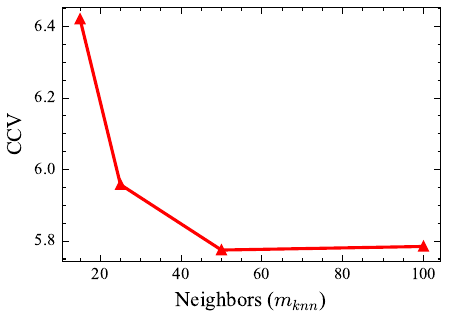}
        \caption{$m_{knn}$: CCV}
    \end{subfigure}
    \vspace{0.3cm}
    
    \begin{subfigure}[b]{0.32\textwidth}
        \includegraphics[width=\linewidth]{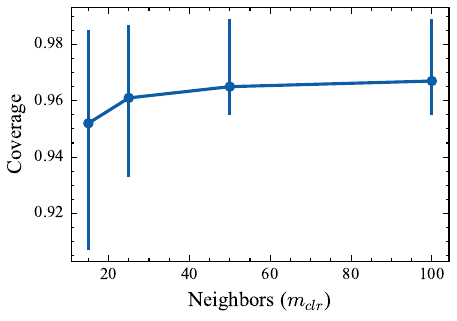}
        \caption{$m_{clr}$: Coverage}
    \end{subfigure}
    \hfill
    \begin{subfigure}[b]{0.32\textwidth}
        \includegraphics[width=\linewidth]{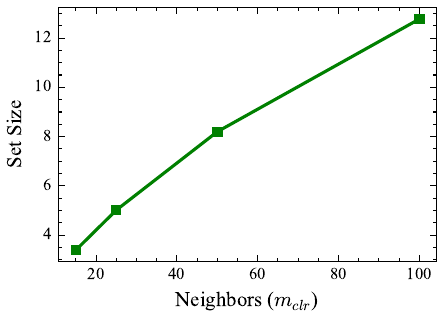}
        \caption{$m_{clr}$: Set Size}
    \end{subfigure}
    \hfill
    \begin{subfigure}[b]{0.32\textwidth}
        \includegraphics[width=\linewidth]{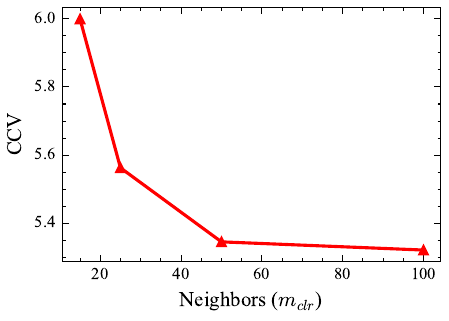}
        \caption{$m_{clr}$: CCV}
    \end{subfigure}
    
    \caption{Additional ablation studies. \textbf{Top Row:} Impact of neighbor number for $m_{knn}$. \textbf{Bottom Row:} Impact of neighbor number for $m_{clr}$}
    \label{fig:appendix_ablations}
\end{figure*}

\subsection{$m_{knn}$ Sensitivity}

We find that performance of $\hat{C}_{knn}(.)$ is generally resistant to choice of $m_{knn}$. We search through $m_{knn} \in [15, 25, 50, 100]$. Figure \ref{fig:appendix_ablations} (top row) shows the effect of $m_{knn}$ on the UQ metrics of coverage, set size and CCV for CLIP ViT-B/16 at $\alpha = 0.05$. As $m_{knn}$ increases, the set size remains steady after around $m_{knn} = 40$ with CCV also remaining steady at just under 5.8. The plots show that for very small neighbor values, the coverage lowers with the minimal marginal coverage drops below the desired 0.95, although this quickly returns to the desired nominal level. This ablation indicates limited benefit of large $m_{knn}$, with it only adding computational burden.

\subsection{$m_{clr}$ Sensitivity}

 Figure \ref{fig:appendix_ablations} (bottom row) shows the effect of $m_{clr}$ on the UQ metrics of coverage, set size and CCV for CLIP ViT-B/16 at $\alpha = 0.05$. We search through $m_{clr} \in [15, 25, 50, 100]$. As for $m_{knn}$, for very small neighbor sizes, the coverage drops below the desired 0.95 level, however this again quickly corrects. From the second plot, we see the choice of $m_{clr}$ highly impacts the resulting set-size where as $m_{clr}$ increases, the set size increases as well. We see the opposite for CCV, where there is a sharp drop in CCV until 50 neighbors. We see that as we go from $m_{clr} = 50$ to $100$, the set size continues to increase; however, the CCV no longer drops, indicating the addition of more neighbors also brings limited benefits.

\clearpage

\section{Selection of $\lambda$} \label{app: lambda_sel}

The hyperparameter $\lambda$ controls the allocation of $\alpha$ between $S_{knn}$ and $S_{clr}$. We determine the optimal $\lambda$ through a grid search over $[0,0.1,...,0.9,1.0]$. First $\mathcal{D}_{\rm cal}$ is partitioned into reference set $\mathcal{D}_{\rm ref}$ (80\%) and tuning set $\mathcal{D}_{\rm val}$ (20\%). For each candidate $\lambda$, we calibrate on $\mathcal{D}_{\rm ref}$ using leave-one-out neighbor search before evaluating prediction sets for $\mathcal{D}_{\rm Test}$. We select the parameter that minimizes a weighted combination of the standardized set size ($Size(\lambda)$), and class-conditional coverage violation ($CCV(\lambda)$) such that  

$$\lambda_{\text{opt}} = \underset{\lambda}{\argmin} (0.8\cdot Size(\lambda) + 0.2\cdot CCV(\lambda) )$$ 

where standardization is computed relative to mean and standard deviation across all candidates. This weighting ensures $\lambda$ is selected prioritizing set size efficiency, while still retaining a level of conditional robustness.

\section{Detailed Results} \label{app: Detailed_Results}

\subsection{RFM and Kernel Parameter Fit Time}

Figure \ref{fig:RFM_Train_Time} shows the computational scalability of the RFM methodology. We notice a linear trend on the $\log$-$\log$ scale, showing a predictable timing quality with respect to data volume. This efficiency is substantial, where the longest fit time for ImageNet required only 713 seconds ($\sim$ 12 minutes), with the majority of datasets requiring under 60 seconds to achieve high predictive performance. We also note that this fit time is for 25 Bayesian Optimization evaluations to fit $\bM$, $\boldsymbol{\beta}$, and determine optimal $\xi, \bandwidth$. This demonstrates that RFM is a computationally efficient solution for task-adaption and kernel fitting.

The BayesianOptimization package \cite{bayesoptim} was used for parameter optimization.

\begin{figure*}[h!]
    \centering
    \includegraphics[width=0.6\linewidth]{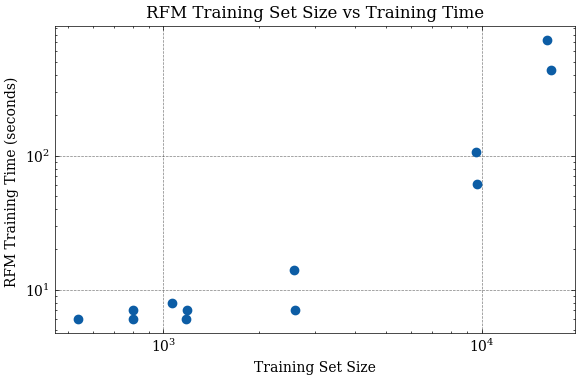}
    \caption{RFM Training Set Size vs Training Time (Both axes are log-scaled.)}
    \label{fig:RFM_Train_Time}
\end{figure*}

\subsection{ViT-B/16 Results} \label{app: Detailed_ViT}

The following show detailed results for CLIP ViT-B/16. Best and second best set size and CCV are \textbf{bolded} and \underline{underlined}, respectively.

\begin{table*}[!h]
\centering
\small
\setlength{\tabcolsep}{4pt}
\renewcommand{\arraystretch}{1}
\resizebox{0.8\textwidth}{!}{%
\begin{tabular}{l r r l r r r r}
\toprule
Dataset & Classes & Test size & Method & Accuracy & Set size & Coverage & CCV \\
\midrule
\multirow{9}{1.2cm}{Imagenet} & \multirow{9}{*}{1000} & \multirow{9}{*}{10000} & DANCE (Ours)       & 0.729 & 10.75 & 0.953 & 6.450 \\
& & & $k$-NN Set (Ours)         & 0.729 & 10.75 & 0.953 & 6.450 \\
& & & CLR Set (Ours)           & 0.729 & 16.18 & 0.957 & \textbf{6.390} \\
& & & Deep $k$-NN         & 0.729 & 11.52 & 0.953 & 6.650 \\
& & & RFM Adapter (APS)   & 0.729 & \textbf{8.04} & 0.951 & 6.500 \\
& & & RFM Adapter (RAPS)  & 0.729 & \underline{8.09} & 0.950 & 6.540 \\
& & & NCP (RAPS)          & 0.729 & 8.88 & 0.950 & \underline{6.400} \\
& & & OT (APS)            & 0.693 & 15.16 & 0.950 & \underline{6.400} \\
& & & OT (RAPS)           & 0.693 & 10.45 & 0.949 & 6.420 \\
\hline
\multirow{9}{1.2cm}{Imagenet-Rendition} & \multirow{9}{*}{200} & \multirow{9}{*}{6000} & DANCE (Ours)            & 0.822 & 6.30 & 0.951 & 4.627 \\
& & & $k$-NN Set (Ours)          & 0.822 & 6.30 & 0.951 & 4.627 \\
& & & CLR Set (Ours)            & 0.822 & 10.29 & 0.958 & 4.181 \\
& & & Deep $k$-NN          & 0.822 & 7.29 & 0.958 & 4.169 \\
& & & RFM Adapter (APS)    & 0.822 & \underline{5.46} & 0.951 & 4.845 \\
& & & RFM Adapter (RAPS)   & 0.822 & 5.61 & 0.952 & 4.783 \\
& & & NCP (RAPS)           & 0.822 & \textbf{3.19} & 0.936 & 5.934 \\
& & & OT (APS)             & 0.795 & 8.68 & 0.953 & \underline{4.117} \\
& & & OT (RAPS)            & 0.795 & 7.23 & 0.953 & \textbf{3.903} \\
\hline
\multirow{9}{1.2cm}{Imagenet-Sketch} & \multirow{9}{*}{1000} & \multirow{9}{*}{10355} & DANCE (Ours)             & 0.718 & \underline{17.78} & 0.954 & 6.032 \\
& & & $k$-NN Set (Ours)         & 0.718 & \textbf{17.09} & 0.953 & 6.077 \\
& & & CLR Set (Ours)           & 0.718 & 19.93 & 0.956 & \underline{5.975} \\
& & & Deep $k$-NN          & 0.718 & 29.02 & 0.968 & \textbf{5.534} \\
& & & RFM Adapter (APS)    & 0.718 & 22.85 & 0.952 & 6.186 \\
& & & RFM Adapter (RAPS)   & 0.718 & 23.00 & 0.952 & 6.175 \\
& & & NCP (RAPS)           & 0.718 & 24.05 & 0.950 & 6.302 \\
& & & OT (APS)             & 0.520 & 43.41 & 0.948 & 6.644 \\
& & & OT (RAPS)            & 0.520 & 29.87 & 0.952 & 6.414 \\
\hline
\multirow{9}{1.2cm}{Eurosat} & \multirow{9}{*}{10} & \multirow{9}{*}{1620} & DANCE (Ours)          & 0.955 & 1.08 & 0.958 & \textbf{2.539} \\
& & & $k$-NN Set (Ours)         & 0.955 & 1.08 & 0.959 & \underline{2.606} \\
& & & CLR Set (Ours)          & 0.955 & 1.22 & 0.956 & 2.872 \\
& & & Deep $k$-NN          & 0.955 & \underline{1.04} & 0.956 & 3.506 \\
& & & RFM Adapter (APS)    & 0.955 & 1.12 & 0.952 & 2.689 \\
& & & RFM Adapter (RAPS)   & 0.955 & 1.12 & 0.952 & 2.689 \\
& & & NCP (RAPS)           & 0.955 & \textbf{1.00} & 0.955 & 2.789 \\
& & & OT (APS)             & 0.559 & 4.63 & 0.952 & 3.078 \\
& & & OT (RAPS)            & 0.559 & 4.58 & 0.951 & 2.956 \\
\hline
\multirow{9}{1.2cm}{Cars} & \multirow{9}{*}{196} & \multirow{9}{*}{1609} & DANCE  (Ours)         & 0.842 & 2.67 & 0.958 & \textbf{5.808} \\
& & & $k$-NN Set(Ours)         & 0.842 & \textbf{2.19} & 0.958 & 6.634 \\
& & & CLR Set (Ours)           & 0.842 & 7.63 & 0.974 & \underline{5.934} \\
& & & Deep $k$-NN          & 0.842 & 3.45 & 0.955 & 6.996 \\
& & & RFM Adapter (APS)    & 0.842 & 2.39 & 0.945 & 7.603 \\
& & & RFM Adapter (RAPS)   & 0.842 & \underline{2.34} & 0.943 & 7.774 \\
& & & NCP (RAPS)           & 0.842 & 3.35 & 0.958 & 6.662 \\
& & & OT (APS)             & 0.692 & 4.39 & 0.958 & 6.662 \\
& & & OT (RAPS)            & 0.692 & 4.22 & 0.958 & 6.675 \\
\hline
\multirow{9}{1.2cm}{Food} & \multirow{9}{*}{101} & \multirow{9}{*}{6060} & DANCE (Ours)         & 0.882 & 1.95 & 0.954 & 2.921 \\
& & & $k$-NN Set (Ours)         & 0.882 & 1.86 & 0.950 & 3.020 \\
& & & CLR Set (Ours)           & 0.882 & 3.17 & 0.955 & 3.102 \\
& & & Deep $k$-NN          & 0.882 & \underline{1.83} & 0.951 & 3.201 \\
& & & RFM Adapter (APS)    & 0.882 & 1.95 & 0.951 & 3.185 \\
& & & RFM Adapter (RAPS)   & 0.882 & 1.98 & 0.952 & 3.267 \\
& & & NCP (RAPS)           & 0.882 & \textbf{1.49} & 0.945 & 3.234 \\
& & & OT (APS)             & 0.853 & 3.18 & 0.953 & \textbf{2.459} \\
& & & OT (RAPS)            & 0.853 & 2.99 & 0.954 & \underline{2.508} \\
\bottomrule
\end{tabular}%
}
\caption{Conformal prediction results (part 1).}
\label{tab:conformal_results_part1}
\end{table*}

\begin{table*}[!h]
\centering
\small
\setlength{\tabcolsep}{4pt}
\renewcommand{\arraystretch}{1}
\resizebox{0.8\textwidth}{!}{%
\begin{tabular}{l r r l r r r r}
\toprule
Dataset & Classes & Test size & Method & Accuracy & Set size & Coverage & CCV \\
\midrule
\multirow{9}{1.2cm}{Pets} & \multirow{9}{*}{37} & \multirow{9}{*}{734} & DANCE (Ours)          & 0.955 & \textbf{1.09} & 0.963 & \textbf{3.514} \\
& & & $k$-NN Set (Ours)        & 0.955 & \textbf{1.09} & 0.966 & \underline{3.784} \\
& & & CLR Set (Ours)          & 0.955 & 2.46 & 0.969 & 3.814 \\
& & & Deep $k$-NN          & 0.955 & 1.37 & 0.970 & 4.219 \\
& & & RFM Adapter (APS)    & 0.955 & 1.20 & 0.950 & 4.369 \\
& & & RFM Adapter (RAPS)   & 0.955 & 1.20 & 0.952 & 4.354 \\
& & & NCP (RAPS)           & 0.955 & \underline{1.11} & 0.966 & 4.069 \\
& & & OT (APS)             & 0.909 & 1.65 & 0.951 & 3.934 \\
& & & OT (RAPS)            & 0.909 & 1.64 & 0.951 & 3.934 \\
\hline
\multirow{9}{1.2cm}{Flowers} & \multirow{9}{*}{102} & \multirow{9}{*}{502} & DANCE (Ours)         & 0.961 & \textbf{1.00} & 0.949 & 9.066 \\
& & & $k$-NN Set (Ours)         & 0.961 & \textbf{1.00} & 0.949 & 9.066 \\
& & & CLR Set (Ours)          & 0.961 & 2.22 & 0.959 & 8.725 \\
& & & Deep $k$-NN          & 0.961 & 7.04 & 0.998 & \textbf{5.147} \\
& & & RFM Adapter (APS)    & 0.961 & 1.24 & 0.957 & 6.938 \\
& & & RFM Adapter (RAPS)   & 0.961 & 1.19 & 0.957 & 6.938 \\
& & & NCP (RAPS)           & 0.961 & \underline{1.18} & 0.985 & \underline{6.019} \\
& & & OT (APS)             & 0.758 & 6.76 & 0.944 & 9.450 \\
& & & OT (RAPS)            & 0.758 & 6.76 & 0.948 & 9.237 \\
\hline
\multirow{9}{1.2cm}{Caltech} & \multirow{9}{*}{101} & \multirow{9}{*}{501} & DANCE (Ours)          & 0.934 & 1.53 & 0.983 & 6.800 \\
& & & $k$-NN Set (Ours)         & 0.934 & \textbf{1.07} & 0.958 & 9.276 \\
& & & CLR Set (Ours)           & 0.934 & 8.71 & 0.989 & \underline{6.100} \\
& & & Deep $k$-NN          & 0.934 & 7.69 & 0.994 & \textbf{5.567} \\
& & & RFM Adapter (APS)    & 0.934 & 1.44 & 0.955 & 9.124 \\
& & & RFM Adapter (RAPS)   & 0.934 & 1.38 & 0.955 & 9.270 \\
& & & NCP (RAPS)           & 0.934 & \underline{1.34} & 0.968 & 8.600 \\
& & & OT (APS)             & 0.925 & 1.82 & 0.964 & 6.978 \\
& & & OT (RAPS)            & 0.925 & 1.74 & 0.962 & 7.078 \\
\hline
\multirow{9}{1.2cm}{DTD} & \multirow{9}{*}{47} & \multirow{9}{*}{338} & DANCE (Ours)           & 0.676 & \textbf{8.99} & 0.965 & \textbf{5.904} \\
& & & $k$-NN Set (Ours)         & 0.676 & \textbf{8.99} & 0.965 & \textbf{5.904} \\
& & & CLR Set (Ours)            & 0.676 & 13.26 & 0.960 & 6.223 \\
& & & Deep $k$-NN          & 0.676 & \underline{10.03} & 0.971 & \underline{6.011} \\
& & & RFM Adapter (APS)    & 0.676 & 11.09 & 0.949 & 7.074 \\
& & & RFM Adapter (RAPS)   & 0.676 & 11.05 & 0.949 & 7.074 \\
& & & NCP (RAPS)           & 0.676 & 17.24 & 0.963 & 6.170 \\
& & & OT (APS)             & 0.471 & 14.12 & 0.955 & 6.968 \\
& & & OT (RAPS)            & 0.471 & 14.02 & 0.955 & 6.968 \\
\hline
\multirow{9}{1.2cm}{UCF} & \multirow{9}{*}{101} & \multirow{9}{*}{744} & DANCE (Ours)           & 0.950 & 1.38 & 0.986 & \textbf{5.342} \\
& & & $k$-NN Set (Ours)        & 0.950 & \textbf{1.08} & 0.971 & 6.195 \\
& & & CLR Set (Ours)            & 0.950 & 5.08 & 0.984 & \underline{5.490} \\
& & & Deep $k$-NN          & 0.950 & 2.22 & 0.960 & 6.914 \\
& & & RFM Adapter (APS)    & 0.950 & 1.31 & 0.956 & 6.917 \\
& & & RFM Adapter (RAPS)   & 0.950 & 1.28 & 0.957 & 6.906 \\
& & & NCP (RAPS)           & 0.950 & \underline{1.19} & 0.967 & 6.682 \\
& & & OT (APS)             & 0.746 & 6.00 & 0.956 & 7.612 \\
& & & OT (RAPS)            & 0.746 & 5.64 & 0.953 & 7.743 \\
\bottomrule
\end{tabular}%
}
\caption{Conformal prediction results (part 2).}
\label{tab:conformal_results_part2}
\end{table*}

\end{document}